\documentclass[letterpaper]{article} 
\usepackage{aaai25}  
\usepackage{times}  
\usepackage{helvet}  
\usepackage{courier}  
\usepackage[hyphens]{url}  
\usepackage{graphicx} 
\usepackage{natbib}  
\usepackage{caption} 
\usepackage{algorithm}
\usepackage{algorithmic}

\usepackage{newfloat}
\usepackage{listings}
\DeclareCaptionStyle{ruled}{labelfont=normalfont,labelsep=colon,strut=off} 
\floatstyle{ruled}
\newfloat{listing}{tb}{lst}{}
\floatname{listing}{Listing}
\usepackage{booktabs}  
\usepackage{pifont}  
\usepackage{mathtools}  
\usepackage{amsfonts}  
\usepackage{subcaption}  
\usepackage{makecell}
\usepackage{multirow}
\usepackage{multicol}
\usepackage{amssymb}  
\usepackage{xcolor}  
\usepackage{soul}  

\usepackage{amsmath}
\DeclareMathOperator*{\argmin}{\arg\,\min}

\usepackage{amsthm}
\newtheorem{theorem}{Theorem}

\newtheorem{corollary}{Corollary}

\newtheorem{definition}{Definition}
\newtheorem{proposition}{Proposition}

\title{Demographic-Agnostic Fairness without Harm}
\author{
    Zhongteng Cai, Mohammad Mahdi Khalili, Xueru Zhang
}
\affiliations{
    Department of Computer Science and Engineering\\
    The Ohio State University \\
    2015 Neil Avenue \\
    Columbus, OH 43210-1277 USA \\
    \{cai.1125, khaliligarekani.1, zhang.12807\}@osu.edu
}

\begin{document}

\maketitle

\begin{abstract}
As machine learning (ML) algorithms are increasingly used in social domains to make predictions about humans, there is a growing concern that these algorithms may exhibit biases against certain social groups. Numerous notions of fairness have been proposed in the literature to measure the unfairness of ML. Among them, one class that receives the most attention is \textit{parity-based}, i.e., achieving fairness by equalizing treatment or outcomes for different social groups. However, achieving parity-based fairness often comes at the cost of lowering model accuracy and is undesirable for many high-stakes domains like healthcare. To avoid inferior accuracy, a line of research focuses on \textit{preference-based} fairness, under which any group of individuals would experience the highest accuracy and collectively prefer the ML outcomes assigned to them if they were given the choice between various sets of outcomes. However, these works assume individual demographic information is known and fully accessible during training. In this paper, we relax this requirement and propose a novel \textit{demographic-agnostic fairness without harm (DAFH)} optimization algorithm, which jointly learns a group classifier that partitions the population into multiple groups and a set of decoupled classifiers associated with these groups. Theoretically, we conduct sample complexity analysis and show that our method can outperform the baselines when demographic information is known and used to train decoupled classifiers. Experiments on both synthetic and real data validate the proposed method.
\end{abstract}

%

\section{Introduction}\label{sec:intro}

Machine learning (ML) and automated decision-making systems trained with real-world data can have inherent bias and exhibit discrimination against certain social groups \cite{barocas2023fairness,NEURIPS2019_7690dd4d,zhang2020long,zhang2020fair,zhang2021fairness}. One common approach to mitigating the issue is to impose certain fairness constraints while making ML predictions. Among all fairness notions proposed in the literature, a class that has received the most attention is based on parity (equality) in treatment or outcomes for different social groups, e.g.,  \textit{demographic parity} \cite{dwork2012fairness} and \textit{equalized odds} \cite{hardt2016equality} that require (true/false) positive rates to be equalized across different groups. However, enforcing these parity-based fairness notions often comes at the cost of lowering model accuracy, resulting in an inherent trade-off between fairness and accuracy~\cite{tradeoff, pang2024fairness,khalili2021fair,khalili2023loss,zuo2023counterfactually,pham2023fairness,zuo2024lookahead,zuo2025post,abroshan2024imposing}. In safety-critical domains such as healthcare, sacrificing predictive accuracy in exchange for fairness is undesirable as it violates both  \textit{beneficence}
(i.e., do what is best for patients) and \textit{non-maleficence} (i.e.,
do no harm) principles required in healthcare ethics \cite{beauchamp1994principles}.

\begin{table*}[tb]
    \centering
    \begin{tabular}{cccccc}
    \toprule
    \textbf{Related Works} & \textbf{Preference-based} & \textbf{No harm} & \textbf{Without sensitive attributes} & \textbf{Intersectionality} \\
    \hline
    \hline
    \citet{preference} & \ding{51} & \ding{51} & \ding{55} & \ding{55} & \\
    \citet{wo_harm} & \ding{51} & \ding{51} & \ding{55} & \ding{51} &  \\
    \citet{adv_reweight} & \ding{55} & \ding{55} & \ding{51} & \ding{51} & \\
    \citet{shared_latent} & \ding{55} & \ding{55} & \ding{51} & \ding{55} & \\
    \citet{pang2024fairness} & \ding{55} & \ding{51} & \ding{51} & \ding{55} & \\
     DAFH (Ours) & \ding{51} & \ding{51} & \ding{51} & \ding{51} & \\
   \bottomrule
\end{tabular}
\caption{Comparison with related works: \citet{preference} and \citet{wo_harm} consider preference-based fairness to avoid harm on specific groups, but both require knowledge of sensitive attributes to generate group partition. \citet{shared_latent,adv_reweight} consider parity-based fairness without accessing sensitive attributes, but they violate no harm constraint. \citet{pang2024fairness} introduces fairness-without-harm when sensitive attributes are inaccessible, but based on parity-based fairness metric. In contrast, our work aims to avoid harm without accessing sensitive attributes, and also deals with the intersectionality caused by the existence of multiple sensitive attributes.}
    \label{tab:related}
\end{table*}

To avoid harming model accuracy, a few studies focus on \textit{preference-based} fairness~\cite{preference, wo_harm}. The idea is to train decoupled ML models for different social groups so that each group of individuals would collectively prefer the model assigned to them to other models. In other words, if individuals were allowed to choose among a set of ML models, a majority of them would still stick to the assigned model. Unlike parity-based fairness, preference-based fairness aims to achieve \textit{fairness without harm} and avoids inferior accuracy for each group.

In this paper, we focus on a preference-based fairness notion similar to the one proposed in \cite{wo_harm}, which requires that the decoupled ML models trained for different social groups satisfy two properties: \textit{rationality} and \textit{envy-freeness}. The former ensures that the accuracy of the decoupled model evaluated on the associated group is higher than that of the pooled model trained with data from all groups, while the latter ensures each group's model has better accuracy than the models assigned to other groups. In addition to \cite{wo_harm}, some subsequent works proposed other variants of preference-based fairness and developed algorithms to train models satisfying the proposed fairness, which we discuss more in Section~\ref{sec:related}.

However, existing methods for achieving preference-based fairness typically assume that individual demographic information (i.e., sensitive attributes such as race and gender) is available during training or can be inferred from side information~\cite{weak_proxy}, so that the decoupled models can be trained on the groups partitioned based on the sensitive attributes \cite{wo_harm}. Unfortunately, individuals' sensitive attributes are not always accessible in practice, either because regulations mandate parity of treatment, prohibiting the use of sensitive attributes in decision-making~\cite{big_data}, or due to user privacy concerns that prevent the disclosure of such attributes. 
For example, when ML models are trained on patient data from multiple hospitals, protected information must be de-identified to comply with regulations such as the Health Insurance Portability and Accountability Act (HIPAA)~\cite{hipaa}. Similarly, in applications like speech or face recognition, sensitive attributes such as race and gender may be unknown, yet it remains desirable to achieve high accuracy across all demographic groups.

Even when sensitive attributes are available, partitioning groups based on them may not lead to more accurate decoupled models. Because individuals are often associated with multiple sensitive attributes such as gender, race, and age, models that are fair with respect to one sensitive attribute can easily violate fairness when assessed on more fine-grained intersectional groups~\cite{prevent_gerry}. Training decoupled classifiers for all intersectional groups is also unrealistic because each intersectional can have limited data samples and the resulting decoupled classifier can easily overfit.  More importantly, sensitive attributes alone cannot adequately capture the full heterogeneity of the population. This motivates a central question: \textit{Can we identify a better group partition under which decoupled models satisfy preference-based fairness, without accessing sensitive attributes?}

To achieve this, we propose a \textit{demographic-agnostic fairness without harm (DAFH)} optimization algorithm that jointly learns (i) a \textit{group classifier} that finds the group partition for fairness without harm constraint (including both rationality and envy-freeness); and (ii) a set of \textit{decoupled classifiers} with each a model assigned to a group partitioned based on the group classifier. Indeed, we can consider our approach as extending the previous methods of partitioning groups based on sensitive attributes to a more complex representation-based partition. Because our approach, by automatically learning individual representations, can more effectively capture population heterogeneity than fixed sensitive attributes. We theoretically show that our method can match or surpass the performance of trivial group partitioning based on sensitive attributes.


A comparison between our approach and representative related works is summarized in Table~\ref{tab:related}. Our main contributions are as follows:

\begin{enumerate}
    \item We propose the measure of fairness without harm when sensitive attributes are inaccessible, which is defined based on \textit{rationality} and \textit{envy-freeness} criteria.
    
    \item We formulate a \textit{demographic-agnostic fairness without harm (DAFM)} optimization algorithm with a novel objective function directly measuring the rationality and envy-freeness, which find group partition and decoupled classifiers that yield higher probability of satisfying fairness without harm.
    \item We theoretically show that our algorithm can better capture the data heterogeneity and achieve the same and even better performance than trivially partitonaing the dataset based on sensitive attribute.
    \item We conduct experiments on synthetic and real datasets, and show that our algorithm can outperform baselines which have access to sensitive attributes.
\end{enumerate}

\section{Related Work}\label{sec:related}

\paragraph{Preference-based fairness.} \citet{preference} propose the first preference-based fairness notion by introducing the concept of envy-freeness from economics~\cite{fair_division, envy}, which implies that no participant prefers the resources assigned to others. The resource is defined as the decoupled classifiers, and the utility of classifiers is often defined as the accuracy of the assigned classifier. \citet{wo_harm} propose using a decision tree to partition the group according to multiple sensitive attributes, until the divided group can best satisfy fairness without harm. \citet{economic} defines similar concepts of group envy-freeness and equitability based on the utility of classification outcomes, and allow groups to be defined arbitrarily, not just based on sensitive attributes.
\citet{fair_ranking} formulate fairness in ranking as a resource allocation problem, where each individual must be assigned a position that satisfies two key criteria: envy-freeness and dominance over a uniform ranking, i.e., no individual will benefit by swapping positions with another or by switching to a uniform random ranking policy.
\citet{perference_survey} provides a survey of preference-based notions of fairness rooted in the literature of computational social choice, including envy-freeness, and emphasize their relevance to various decision-making applications such as recommendation and classification.


\paragraph{Fairness without demographic data.} Achieving fairness without accessing sensitive attributes is crucial when individual demographic data is not readily available~\cite{demo_challenge}. However, the research on this topic has been largely focused on parity-based fairness~\cite{wo_demo_survey}; this is fundamentally different from our work which aims to achieve preference-based fairness without accessing sensitive attributes. To attain fairness in the absence of demographic data, many studies attempt to infer sensitive attributes, e.g., by learning proxy models from a separate training dataset that includes demographic data, in order to predict the sensitive attributes in the target dataset for downstream tasks like debiasing~\cite{cvae_fairness, shared_latent}. However, empirical results have shown that relying on proxy variables to measure fairness can lead to biased outcomes~\cite{weak_proxy, eval_fairness}.  
Other approaches minimize the risk of disparity when group partitions are unknown~\cite{adv_reweight, dro_fairness},
akin to our method, which also does not require an additional dataset containing demographic information.
\citet{pang2024fairness} proposed a parity-based fairness without harm, where new data samples without sensitive attributes are actively acquired during training to mitigate fairness parity. While they avoid harming accuracy by filtering fresh samples, we prevent harm by assigning decoupled classifiers to each group.
Some works leverage properties of representations or performance on downstream tasks to identify group partition. \citet{george} increases the worst-case accuracy for subclasses with unknown labels by leveraging the separable representations in the feature space. \citet{antigone} uses incorrectly classified samples as proxies for disadvantaged groups. While these works focus on parity-based fairness, their method can be integrated with our framework to guide the design of our group classifier.

\section{Problem Statement}\label{sec:problem}

Consider a dataset $\mathcal{D} = \{ (x_i, y_i)\}_{i=1}^n $ with $n$ independent samples drawn from the same joint distribution $(x_i, y_i) \sim P(X, Y)$, where $x_i \in \mathcal{X}$ are individual's features and $y_i \in \mathcal{Y} \coloneq \{\pm 1 \}$ is the binary target variable to be predicted. We consider practical settings where individuals' demographic information (i.e., sensitive attributes such as race, gender, and age) are unknown and not included in $\mathcal{D}$.

Since individuals' sensitive attributes are unknown, we introduce a \textit{group classifier} $\theta: \mathcal{X} \to \mathcal{K}$ from hypothesis class $\mathcal{H^G}$ that maps the input features $x_i$ to group label $k \in \mathcal{K}\coloneq [K]$. For samples assigned to each group $k$, let $h_k: \mathcal{X} \to \mathcal{Y}$ be the corresponding \textit{decoupled classifier} from hypothesis class $\mathcal{H^D}$ that predicts the target variable based on the input feature $X$. 
Given a classifier $h:\mathcal{X} \to \mathcal{Y}$, its empirical risk $\hat{R}(h)$ and true risk ${R}(h)$ associated with 0-1 loss are defined respectively as follows:
\begin{eqnarray*}
\hat{R}(h) &=& \frac{1}{n} \sum_{i=1}^n \mathbb{I}(h(x_i) \neq y_i);\\ {R}(h) &=& \mathbb{E}_{X,Y}[\mathbb{I}(h(X) \neq Y)],    
\end{eqnarray*}
where $\mathbb{I}$ is the indicator function. Further define the \textit{pooled classifier} $h_0$ as the classifier from $\mathcal{H^D}$ that minimizes the true risk, i.e., $h_0 = \argmin_{h} R(h)$.  Given a classifier $h:\mathcal{X} \to \mathcal{Y}$, define the group-specific empirical risk as the empirical risk evaluated on samples assigned to group $k$, i.e., $$\hat{R}_k(h) = \frac{1}{n_k} \sum_{i=1}^n \Big[ \mathbb{I}(\theta(x_i)=k) \cdot \mathbb{I}(h(x_i) \neq y_i) \Big],$$
where $n_k$ is the number of samples assigned to group $k$.
Likewise, denote the group-specific true risk of group $k$ as $${R}_k(h) = \mathbb{E}_{X,Y | \theta(X)=k} \left[\mathbb{I}(h(X) \neq Y) \right].$$ Our goal is to train the group classifier $\theta$ and a set of decoupled classifiers $\{h_k\}_{k \in [K]}$ that satisfy fairness without harm requirement, as defined below.
\begin{definition}[Fairness without harm]\label{def} A group classifier $\theta$ and decoupled classifiers $\{h_k\}_{k \in [K]}$ satisfy fairness without harm if the resulting group-specific true risks satisfy both \textbf{rationality} and \textbf{envy-freeness} properties:  
\begin{enumerate}
    \item \textbf{Rationality:} $R_k(h_k) \leq R_k(h_0)$, $\forall k \in [K]$.
    \item \textbf{Envy-freeness:} $R_k(h_k) \leq R_k(h_j)$, $\forall k, j \in [K]$.
\end{enumerate}
\end{definition}

Definition~\ref{def} implies that for individuals in each group $k$, the true risk is lowest under the associated decoupled classifier $h_k$. In other words, if individuals are given opportunities to choose their own classifiers from $h_0$ and $\{h_k\}_{k \in [K]}$, a \textit{majority} of individuals in each group will prefer their assigned decoupled classifiers to the pooled classifier (rationality) and the decoupled classifiers assigned to other groups (envy-freeness), given their preference to lower true risk. We adopt the same definitions of rationality and envy-freeness as presented in~\cite{wo_harm}, with one key distinction: we define these notions over group partitions induced by a learned group classifier, rather than by sensitive attributes. It is worth noting that our definition does not require every individual to prefer their assigned classifier.

We approximate the degree of \textit{rationality} violation by evaluating the difference in empirical risk between the pooled classifier and the assigned classifier for each group. Similarly, we estimate the degree of \textit{envy-freeness} violation by examining the difference in empirical risk between the decoupled classifiers for each group. The goal is to find  group classifier $\theta$ and decoupled classifiers $\{h_k\}_{k \in [K]}$ that maximize the probability that individuals satisfy rationality or envy-freeness, as defined in Definition~\ref{def}. We formulate this as the following optimization:
\begin{align}\label{equ:obj0}
\resizebox{.9\hsize}{!}{$\displaystyle      \underset{\theta, \{h_k\}}{\max\,\,} \frac{1}{K} \sum_{k=1}^K \biggl\{\underset{\text{Rationale}}{\underbrace{\left[ \hat{R}_k(\hat{h}_0) - \hat{R}_k(h_k) \right]}} + \frac{1}{K} {\underset{\text{Envy-freeness}}{\underbrace{\sum_{j =1}^{K} \left[ \hat{R}_k(h_j) - \hat{R}_k(h_k) \right]}}}\biggr\},$}
\end{align}
where $\hat{h}_0 = \argmin_{h} \,\hat{R}(h)$ is the classifier minimizing the empirical risk over all samples. Although the objective function \eqref{equ:obj0} is a direct combination of two criteria (rationality and envy-freeness), optimizing it is non-trivial due to the following challenges: (i) 
Objective \eqref{equ:obj0} is a function of parameters in group classifier and decoupled classifiers; both must be trained jointly since the samples used for training each decoupled classifier depend on the outputs of the group classifier. (ii) Objective \eqref{equ:obj0} is a highly non-convex and non-differentiable function, as it involves differences in empirical risks across multiple pairs of groups.


\section{Proposed Algorithm}\label{sec:alg}
Next, we introduce our approach that solves optimization \eqref{equ:obj0}. 
We first derive a lower bound on the objective function \eqref{equ:obj0} that is easier to optimize using the training data. Then, we approximate this new objective function using differentiable functions and find the optimal group classifier $\theta$ and decoupled classifiers $\{h_k\}_{k \in [K]}$ iteratively using gradient ascents. 

\paragraph{Simplify objective function \eqref{equ:obj0} with a lower bound.}

Given group classifier $\theta:\mathcal{X}\to [K]$ and training dataset $\mathcal{D}=\{(x_i, y_i)\}_{i=1}^n$, denote $\pi_i^k \coloneq \mathbb{I}(\theta(x_i)=k)$ as the indicator of group assignment for $i$-th sample. We have $\sum_{k=1}^{K} \pi_i^k=1, \forall i$. 
Let $L_i^0 \coloneq \mathbb{I}(h_0(x_i) \ne y_i)$ and $L_i^k \coloneq \mathbb{I}(h_k(x_i) \ne y_i)$ denote the losses of the sample $(x_i, y_i)$ under pooled classifier $h_0$ and decoupled classifier $h_k$, respectively. Then we can have the following proposition:

\begin{proposition}\label{thm:obj}
By replacing $n_k$ (the number of samples assigned to group $k$ by $\theta$) in \eqref{equ:obj0} with $n$ (the total number of samples), we obtain the the following lower bound on the objective function \eqref{equ:obj0}: $$\underline{\mathrm{obj}}+\mathrm{constant~terms}$$ where we define
    \begin{align}\label{equ:obj}
        \underline{\mathrm{obj}} = \frac{1}{nK^2}\sum_{i=1}^n \sum_{k=1}^K \left(  L_i^k -  2K \pi_i^k L_i^k   \right)
    \end{align}
\end{proposition}
\begin{proof}
The objective function in~\eqref{equ:obj0} can be equivalently written as: 
$$
    \frac{1}{K} \sum_{k=1}^K  \Bigl[ \frac{1}{n_k}{ \sum_{i=1}^{n} \pi_i^k (L_i^0-L_i^k) } + \frac{1}{K} \sum_{j=1}^K \frac{1}{n_k} \sum_{i=1}^{n} \pi_i^k ( L_i^j - L_i^k ) \Bigr].
$$
Since the decoupled classifiers in each group can achieve at least the same accuracy as the pooled classifier and decoupled classifiers assigned to other groups, assume each group satisfies rationale and envy-freeness, therefore it has the following lower bound:

    \begin{align*}
    & \frac{1}{K} \sum_{k=1}^K  \Bigl[ \frac{1}{n}{ \sum_{i=1}^{n} \pi_i^k (L_i^0-L_i^k) } + \frac{1}{K} \sum_{j=1}^K \frac{1}{n} \sum_{i=1}^{n} \pi_i^k ( L_i^j - L_i^k ) \Bigr] \\
     &  = \frac{1}{n} \sum_{i=1}^n \Biggl\{ \frac{1}{K} \sum_{k=1}^K \Bigl[ {  \pi_i^k (L_i^0-L_i^k) } + \frac{1}{K} \sum_{j=1}^K {   \pi_i^k ( L_i^j - L_i^k )} \Bigr] \Biggr\} \notag\\
    & = \frac{1}{Kn} \sum_{i=1}^n \sum_{k=1}^K \Bigl( \pi_i^k L_i^0 - 2 \pi_i^k L_i^k + \frac{1}{K} \pi_i^k \sum_{j=1}^K L_i^j \Bigr) \notag\\
     & = \frac{1}{Kn} \sum_{i=1}^n L_i^0 + \frac{1}{Kn} \sum_{i=1}^n \left(  - 2 \sum_{k=1}^K  \pi_i^k L_i^k + \frac{1}{K} \sum_{j=1}^{K} L_i^j \right) \\
     & = \underbrace{\frac{1}{Kn} \sum_{i=1}^n L_i^0}_{\text{constant term}} + \underbrace{\frac{1}{nK^2}\sum_{i=1}^n \sum_{k=1}^K \left(  L_i^k -  2K \pi_i^k L_i^k   \right)}_{\underline{\mathrm{obj}}}  \notag
\end{align*}

Then we can prove the proposition.

\end{proof}
It is worth noting that we replace $n_k$ in \eqref{equ:obj0} with $n$ to improve optimization stability. Since $n_k$ appears in the denominator of \eqref{equ:obj0}, it introduces non-differentiability and requires approximation---an issue that becomes particularly unstable when group sizes are small. Therefore, rather than directly optimizing the objective~\eqref{equ:obj0}, we optimize the surrogate objective~\eqref{equ:obj}, which serves as a valid and more stable proxy.


    
\paragraph{Discussion on objective function.} It is worth noting that the objective function in \eqref{equ:obj} is related but different from the overall accuracy. Specifically, let the overall accuracy of decoupled classifiers on the assigned groups be defined as:\
\begin{align}\label{equ:acc}   \mathrm{acc}  = \frac{1}{n} \sum_{i=1}^n \sum_{k=1}^K \pi_i^k (1 - L_i^k).
\end{align}
Then, maximizing \eqref{equ:obj} is equal to maximizing the following: 
\begin{align}
    \underline{\mathrm{obj}} = \frac{2}{K} \mathrm{acc}  + \frac{1}{nK^2} \sum_{i=1}^n \sum_{k=1}^K L_i^k - \frac{2}{K}
\end{align}
Compared with the overall accuracy defined in \eqref{equ:acc}, our objective function \eqref{equ:obj} also considers the performance of decoupled classifiers on groups that they are not assigned to. 
As we will show in Section~\ref{sec:thm} and \ref{sec:exp}, the group classifier $\theta$ found under our method better captures the data heterogeneity and can result in more accurate decoupled classifiers. 


\paragraph{Approximation using differentiable functions.} We further use differentiable functions to approximate~\eqref{equ:obj}. Specifically, instead of hard group indicator $\pi_i^k=\mathbb{I}(\theta(x_i)=k)$, we consider soft assignment and define $\widetilde{\pi}_i^k=\Pr(\theta(x_i)=k)$ as the predicted probability that a sample belonging to group $k$. Let group classifier $\theta$ be a neural network which softmax layer its outputs, then we use $k$-th output of its softmax layer (denoted as $\theta(x)[k]$) to compute $\widetilde{\pi}_i^k$. Similarly, instead of 0-1 loss $L_i^k$, we use $h_k(x)[1]$ to denote the predicted probability that a sample $x$ has label 1 and approximate 0-1 loss with a Sigmoid function $ \widetilde{L}_i^k$ as follows.
\begin{align}
    \widetilde{L}_i^k & =  \mathrm{Sigmoid} \left(h_k(x_i)[1] - \frac{1}{2}\right) - y_i , \notag\\
    \widetilde{\pi}_i^k & =  \theta(x_i)[k]. \notag
\end{align}
During the empirical studies in Section \ref{sec:exp}, we observe that the group classifier learned based on $\widetilde{\pi}_i^k, \widetilde{L}_i^k$ tends to assign a majority of training samples to the same group. To enforce balanced group size, we add a penalty term to the objective function, which is calculated as the KL-divergence between the group classifier output and uniform distribution, denoted as $\Gamma(\theta)$, times a hyperparameter $\lambda$. 
Combine the above, we have the following:
\begin{align}\label{equ:calculate}
    \widetilde{\mathrm{obj}} = \frac{1}{nK^2} \sum_{i=1}^n \sum_{i=1}^K \left( \widetilde{L}_i^k  - 2K \widetilde{\pi}_i^k \widetilde{L}_i^k \right) + \lambda \Gamma(\theta).
\end{align}
Since $\widetilde{\mathrm{obj}}$ is a differentiable function of parameters in $\theta$ and $h_k$, we can find optimal parameters that yield the highest degree of fairness without harm using stochastic gradient ascent as shown in Algorithm~\ref{alg:obj}. 
Here, we slightly abuse the notation by using $\theta$ and $h_k$ denote the model parameters associated with group classifier and decoupled classifier. Let $\widetilde{\mathrm{obj}}(\theta, \{h_k\}_{k\in[K]}, x_i, y_i, \lambda)$  be calculated as in \eqref{equ:calculate}.

\begin{algorithm}[t]
    \caption{Stochastic gradient ascent to optimize $\widetilde{\mathrm{obj}}$}
    \label{alg:obj}
    \begin{algorithmic}[1]
        \STATE {\bfseries Input:} Training data $\mathcal{D} = \{ (x_i, y_i)\}_{i=1}^n $, differentiable function $\widetilde{\mathrm{obj}}$, learning rate $\alpha$, hyperparameter $\lambda$
        \STATE Initialize group and decoupled  classifiers $\theta$, $\{h_k\}_{k\in[K]}$
        \WHILE{$\mathrm{obj}$ not converge}
        \FOR{$i = 1, \cdots, n$}
        \STATE $\theta \gets \theta + \alpha \nabla_{\theta} ~ \widetilde{\mathrm{obj}}(\theta, \{h_k\}_{k\in[K]}, x_i, y_i, \lambda)$
        \FOR{$k = 1, \cdots, K$}
        \STATE $h_k \gets h_k + \alpha \nabla_{h_k} ~ \widetilde{\mathrm{obj}}(\theta, \{h_k\}_{k\in[K]}, x_i, y_i, \lambda)$
        \ENDFOR
        \ENDFOR
        \ENDWHILE
        \STATE {\bfseries Output:} $\theta, \{h_k\}_{k\in[K]}$
    \end{algorithmic}
\end{algorithm}
\section{Theoretical Analysis}\label{sec:thm}
This section theoretically analyzes the performance of the proposed method. While our algorithm is trained by maximizing empirical performance on training data, we conduct analysis to show the generalization ability. We also show that $\{h_k\}_{k\in[K]}$ trained under our method are guaranteed to outperform decoupled classifiers trained by partitioning groups based on sensitive attributes. Unless otherwise stated, we focus on 
hyperparameter $\lambda=0$ in~\eqref{equ:calculate}. 

\subsection{Generalization} We first show that the difference between objective functions calculated through empirical risks and true risks is upper bounded. Define the true degree of fairness without harm as: 
\begin{align*}
\resizebox{.98\hsize}{!}{$\displaystyle      {\mathrm{obj}} = \frac{1}{K} \sum_{k=1}^K \biggl\{\left[ {R}_k({h}_0) - {R}_k(h_k) \right] + \frac{1}{K} \sum_{j =1}^{K} \left[ {R}_k(h_j) - {R}_k(h_k) \right]\biggr\},$}
\end{align*}
while the empirical degree of fairness is the objective function of \eqref{equ:obj0} defined as: 
\begin{align*}
\resizebox{0.98\hsize}{!}{$\displaystyle      {\widehat{\mathrm{obj}}} = \frac{1}{K} \sum_{k=1}^K \biggl\{\left[ \hat{R}_k({h}_0) - \hat{R}_k(h_k) \right] + \frac{1}{K} \sum_{j =1}^{K} \left[ \hat{R}_k(h_j) - \hat{R}_k(h_k) \right]\biggr\},$}
\end{align*}

Assume that the group size decided by the group classifier has a lower bound, i.e. $\forall k \in [K], n_k \geq \underline{n}$. We now show that the difference between $\mathrm{obj}$ and $\mathrm{\widehat{\mathrm{obj}}}$ is upper bounded.

\begin{theorem}\label{thm:1}
   Let $n_k$ be the number of samples assigned to group $k$ under group classifier $\theta$  and denote the lower bound of group size as $\underline{n}$. For any $\epsilon > 0$, with probability at least $1-2K(K+1) \exp\{-2 \epsilon^2 \underline{n} / 9\}$, we have
\begin{align*}
    \left|\mathrm{obj} - \widehat{\mathrm{obj}}\right| \leq \epsilon.
\end{align*}
\end{theorem}
\begin{proof}
For any hypothesis $h \in \mathcal{H^D}$, denote $|R_k(h)-\hat{R}_k(h)|$ as $\Delta_k(h)$. Then we have:
\begin{align*}
\resizebox{.98\hsize}{!}{$\displaystyle 
\left|{\mathrm{obj}} - \widehat{\mathrm{obj}}\right| \leq \,  \frac{1}{K} \sum_{k=1}^K \biggl[ \Delta_k(h_0) + 2\Delta_k(h_k) + \frac{1}{K} \sum_{j =1}^{K} \Delta_k(h_j) \biggr]$}
\end{align*}
According to Hoeffding’s Inequality~\cite{hoeffding}, for any $\epsilon > 0$, we have:
$$
\Pr \left[\Delta_k(h) > \epsilon \right] \leq 2 \exp \{-2 \epsilon^2 n_k\} \leq 2 \exp \{-2 \epsilon^2 \underline{n}\}.
$$
$\forall k \in [K]$, $j \in \{0, \cdots, K\}$, denote $\Delta_k(h_j) > \epsilon$ as event $A_{kj}$, then we have the following:
\begin{align*}
    & \Pr \left [\bigcap_{k,j} \overline{A_{kj}} \right ] = \, \Pr \left [\overline{\bigcup_{k,j} A_{kj}}\right ] = \, 1 - \Pr \left [\bigcup_{k,j} A_{kj}\right ] \\
    \geq & \, 1 - \sum_{k=1}^K \sum_{j=0}^K \Pr \left[ A_{kj} \right] \geq \, 1-2K(K+1) \exp\{-2 \epsilon^2 \underline{n}\}.   
\end{align*}
Event $\bigcap_{k,j} \overline{A_{kj}}$ means that $\forall k,j$, $\Delta_k(h_j) \leq \epsilon$, which indicates that $|{\mathrm{obj}} - \widehat{\mathrm{obj}}| \leq 3\epsilon$, hence we can know that:
\begin{align*}
  \Pr \left[\left|{\mathrm{obj}} - \widehat{\mathrm{obj}}\right| \leq 3\epsilon \right] &\geq \Pr \left[ \bigcap_{k,j} \overline{A_{kj}} \right] \\ &\geq 1-2K(K+1) \exp\{-2 \epsilon^2 \underline{n}\}.  
\end{align*}
Now we can prove the theorem by replacing $\epsilon$ with $\frac{\epsilon}{3}$.

\end{proof}


\subsection{Feasibility} Suppose each individual has a hidden sensitive attribute $S\in\mathcal{S}$ indicating the demographic information (e.g., race, gender, age). We then prove that our algorithm can achieve at least the same performance as directly using the sensitive attribute $S$ to divide the dataset. 

Specifically, denote $P_k \coloneqq \Pr(X,Y|\theta(X)=k)$ as the data distribution of group $k$ determined under group classifier $\theta$, and let $P_k^{*}\coloneqq \Pr(X,Y|S=k)$ be the distribution of group $k$ partitioned based on sensitive attribute $S$. 
For any two classifiers $h, h^{\prime} \in \mathcal{H^D}$, let the difference between their performance on $P_k$ and $P^*_k$ be defined as $D_k(h, h^{\prime}) \coloneqq \mathbb{E}_{X \thicksim P_k}\left[ \mathbb{I}(h(X) \neq h^{\prime}(X))\right]$ and $D^*_k(h, h^{\prime}) \coloneqq \mathbb{E}_{X \thicksim P^*_k}\left[ \mathbb{I}(h(X) \neq h^{\prime}(X))\right]$, respectively.

Suppose the population is partitioned based on the sensitive attribute $S$, and define $\{h^*_k\}_{k \in[K]}$ as the set of classifiers that attains the maximal true degree of fairness without harm $\mathrm{{obj}_{P^*}}$ under $\{P^*_k\}_{k\in[K]}$: 
\begin{align*}
\resizebox{.98\hsize}{!}{$\displaystyle 
\mathrm{{obj}_{P^*}} = \frac{1}{K} \sum_{k=1}^K \biggl\{ \left[ {R}^*_k({h}_0) - {R}^*_k(h^*_k) \right] + \frac{1}{K} \sum_{j =1}^{K} \left[ {R}^*_k(h^*_j) - {R}^*_k(h^*_k) \right]\biggr\}$}.
\end{align*}
Denote the maximal objective value of our classifiers $\{h_k\}_{k \in[K]}$ assessed on  $\{P_k\}_{k\in[K]}$ derived by our group classifier as:
\begin{align*}
\resizebox{.98\hsize}{!}{$\displaystyle 
\mathrm{{obj}_{P}} = \frac{1}{K} \sum_{k=1}^K \biggl\{ \left[ {R}_k({h}_0) - {R}_k(h_k) \right] + \frac{1}{K} \sum_{j =1}^{K} \left[ {R}_k(h_j) - {R}_k(h_k) \right]\biggr\}$}
\end{align*}

\begin{theorem}\label{thm:2} 
Assume the loss function $\ell$ is symmetric and obeys triangle inequality, i.e., $\forall y_1, y_2, y_3 \in \mathcal{Y}, \ell(y_1, y_2) \leq \ell(y_1, y_3) + \ell(y_3, y_2)$, and for every pair of $P^*_k$ and $P_k$, there exists $h \in \mathcal{H^D}$, such that $R_k^*(h) = R_k(h)$, then we have:
$$\mathrm{obj_{P^*}} - \mathrm{obj_P} \leq 3 \, \underset{k}{\max} \, \mathrm{disc} (P_k, P_k^*),$$
where $\mathrm{disc} (P_k, P_k^*)$ measures the discrepancy between two distributions $P_k$ and $P_k^{*}$ and is defined as:
$$
\mathrm{disc}(P_k, P_k^{*}) = \max_{h, h^{\prime} \in \mathcal{H^D}}  |D_k(h, h^{\prime}) - D_k^*(h, h^{\prime})|.
$$
\end{theorem}
\begin{proof}
    
Define $f: \mathcal{X} \to \mathcal{Y}$ as the labeling function, i.e., for each sample $(x,y)$, we have $f(x) = y$, then true risk $R_k(h)$ can be represented as $D_k(h, f)$. Define group-specific true risk on distribution $P_k^*$ as ${R}^*_k(h) = D^*_k(h,f) = \mathbb{E}_{X,Y |S=k} \left[\ell(h(X), Y) \right]$. Without lossing generality, we can assume $D_k(h, f) \geq D_k^*(h,f)$. For any classifier $h \in \mathcal{H^D}$ and group index $k$, according to the triangle inequality of the loss function, we have:

\begin{align*}
    & |R_k(h) - R_k^*(h)| = |D_k(h, f) - D_k^*(h, f)| \\
    \leq & \, |\left[ D_k(h, h^{\prime}) + D_k(h^{\prime}, f) \right] - \left[  D_k^*(h^{\prime}, f) - D_k^*(h, h^{\prime}) \right]| \\
    \leq & \, \left| D_k(h, h^{\prime}) - D_k^*(h, h^{\prime}) \right| + \left| D_k(h^{\prime}, f) - D_k^*(h^{\prime}, f) \right|, 
\end{align*}

where $h^{\prime}$ can be any hypothesis in $\mathcal{H^D}$. Since there exist $h^{\prime}$ such that $D_k^*(h^{\prime}, f) = D_k(h^{\prime}, f)$, and according to the definition of discrepancy distance, we can have:
$$
|D_k(h, f) - D_k^*(h, f)| \leq \mathrm{disc}(P_k, P_k^{*}).
$$
Let $\{h_k^*\}$ be the set of classifiers that attains the maximal true degree of fairness without harm $\mathrm{obj_{P^*}}$ under distribution $\{P_k^*\}_{k \in [K]}$.
Let $\mathrm{obj_{P}^{\prime}}$ be the true degree of fairness attained with the same set of classifiers $\{h_k^*\}$ under distribution $\{P_k\}_{k \in [K]}$ determined by our group classifier, i.e.,

\resizebox{.98\hsize}{!}
{$\displaystyle
$$\displaystyle\mathrm{{obj}_{P}^\prime} = \frac{1}{K} \sum_{k=1}^K \biggl\{ \left[ {R}_k({h}_0) - {R}_k(h_k^*) \right] + \frac{1}{K} \sum_{j =1}^{K} \left[ {R}_k(h_j^*) - {R}_k(h_k^*) \right]\biggr\}.$$ 
$}

According to our definition, $\mathrm{obj_P} \geq \mathrm{obj_P^\prime}.$
Then we can know that the difference between  $\mathrm{obj_{P^*}}$ and $\mathrm{obj_P}$ is upper bounded:
\begin{align*}
    \mathrm{obj_{P^*}} - \mathrm{obj_P}
    &\leq \mathrm{obj_{P^*}} - \mathrm{obj_P^\prime} \\
    &\leq 3 \, \underset{k}{\max} \, \mathrm{disc} (P_k, P_k^*)
\end{align*}


\end{proof}

Given a hypothesis class, distance metric $\mathrm{disc}(\cdot,\cdot)$ measures the difference between two distributions: with a higher distance, it is more likely that the same hypothesis has diverse outputs~\cite{distance}. 

Theorem~\ref{thm:2} shows that the performance gap between our method and the optimal models trained by partitioning groups based on sensitive attributes is bounded by the difference between $P_k$ and $P_k^{*}$. Suppose there is a ground truth function $\phi:\mathcal{X} \to \mathcal{S}$ such that $\phi(x)=s$ and $\phi\in \mathcal{H^G}$, the following result is straightforward.

\begin{corollary}
If $\phi\in \mathcal{H^G}$ and we can find a group classifier $\theta$ such that $\mathrm{disc}(P_k, P_k^{*})=0,\forall k$, then $\mathrm{obj_P} \geq \mathrm{obj_{P^*}}$.
\end{corollary}



\section{Experiment}\label{sec:exp}

\begin{table}[tb]
    \centering
    \small
    \begin{tabular}{ll}
    \toprule
      \textbf{Hyperparameter} & \textbf{Values} \\
      \hline
     Batch size & 256, 1,024 \\
     Training iterations & 2, 3, 4\\
     Learning rate for group classifier & $10^{-4}$, $10^{-3}$\\
     Learning rate for decoupled classifiers & $10^{-2}$, $10^{-1}$\\
     Momentum for decoupled classifiers & 0.9\\
     Penalty parameter $\lambda$ & 0, 1, 10, 100 \\
   \bottomrule
\end{tabular}
\caption{Hyper-parameters used in experiments}
\label{tab:parameter}
\end{table}

\begin{table*}[tb]
    \centering
    \small
    \begin{tabular}{lccccc}
    \toprule
      \textbf{Hyperparameter} & \textbf{Adult} & \textbf{Arrest} & \textbf{Violent} & \textbf{German} & \textbf{Bank} \\
      \hline
     Batch size & 1,024 & 1,024 & 1,024 & 256 & 1,024\\
     Training epochs & 3 & 3 & 3 & 3 & 2 \\
     Learning rate for group classifier & $10^{-3}$ & $10^{-3}$ & $10^{-3}$ & $10^{-3}$ & $10^{-3}$ \\
     Learning rate for decoupled classifiers & $10^{-2}$ & $10^{-2}$ & $10^{-2}$ & $10^{-2}$ & $10^{-2}$ \\
     Momentum for decoupled classifiers & 0.9 & 0.9 & 0.9 & 0.9 & 0.9 \\
     Penalty parameter $\lambda$ & 10 & 10 & 10 & 100 & 10 \\
   \bottomrule
\end{tabular}
\caption{Final hyper-parameters selected for each dataset}
\label{tab:parameter_ds}
\end{table*}

\begin{figure*}[tb]
\centering
\begin{subfigure}[t]{0.23\linewidth}
    \centering
    \includegraphics[width=\linewidth]{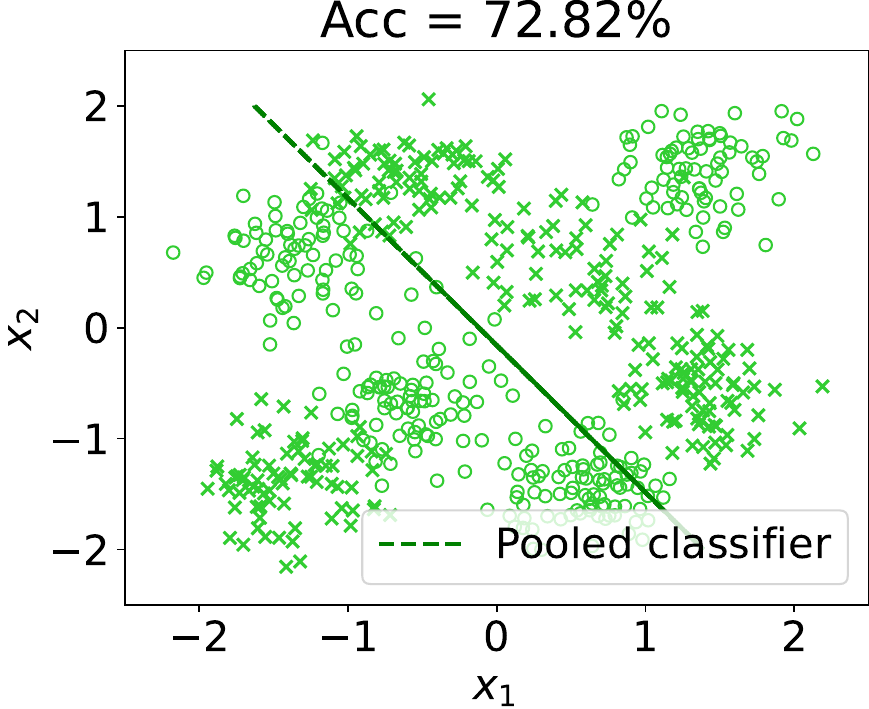}
    \caption{Without partition}
    \label{fig:syn_pol}
\end{subfigure}
\hfill
\begin{subfigure}[t]{0.23\linewidth}
    \centering
    \includegraphics[width=\linewidth]{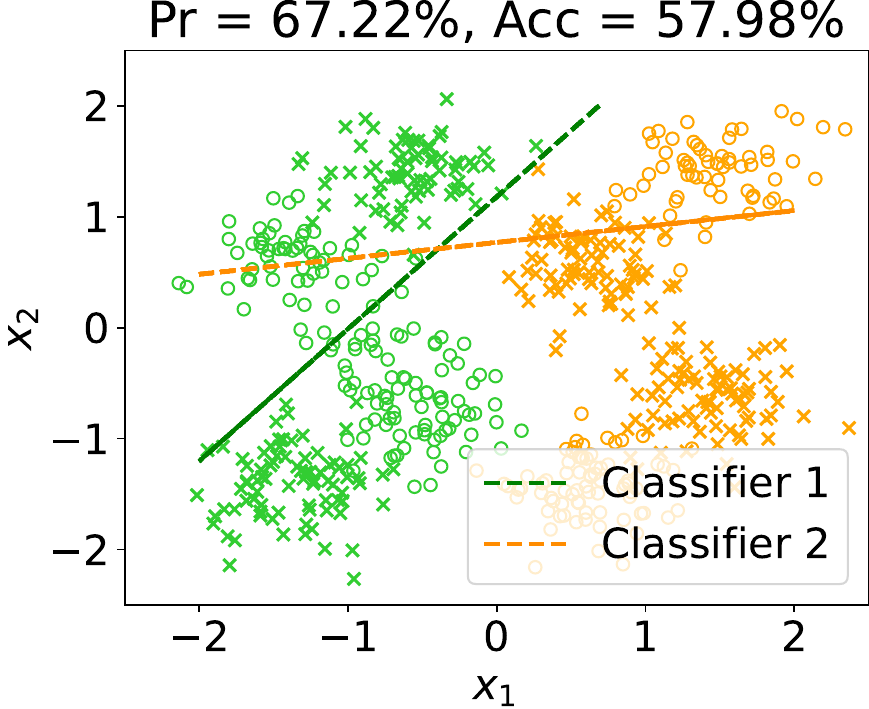}
    \caption{Partition by attribute $s_1$}
    \label{fig:syn_tp1}
\end{subfigure}
\hfill
\begin{subfigure}[t]{0.23\linewidth}
    \centering
    \includegraphics[width=\linewidth]{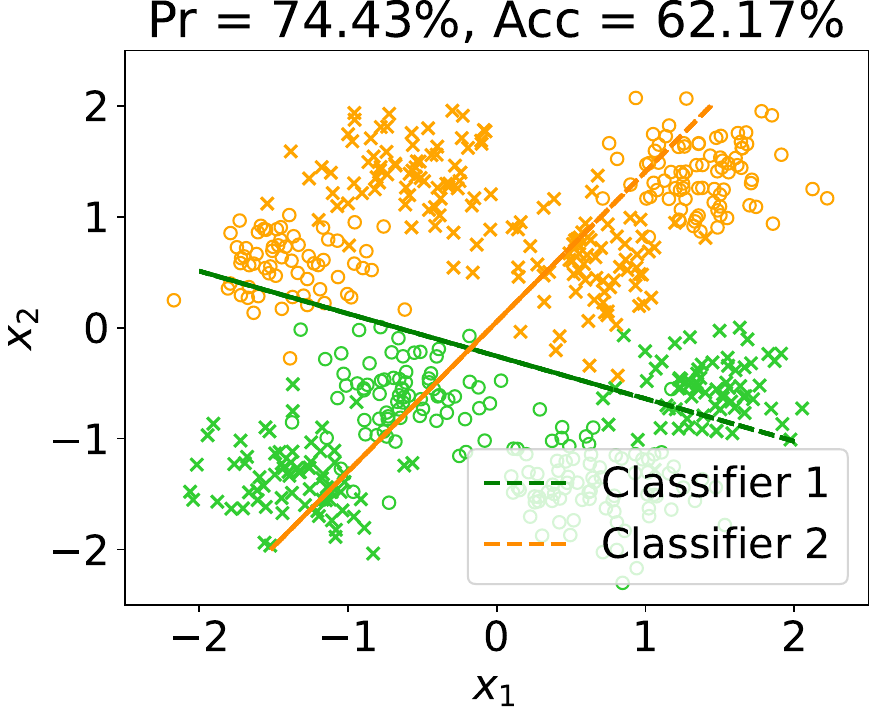}
    \caption{Partition by attribute $s_2$}
    \label{fig:syn_tp2}
\end{subfigure}
\hfill
\begin{subfigure}[t]{0.23\linewidth}
    \centering
    \includegraphics[width=\linewidth]{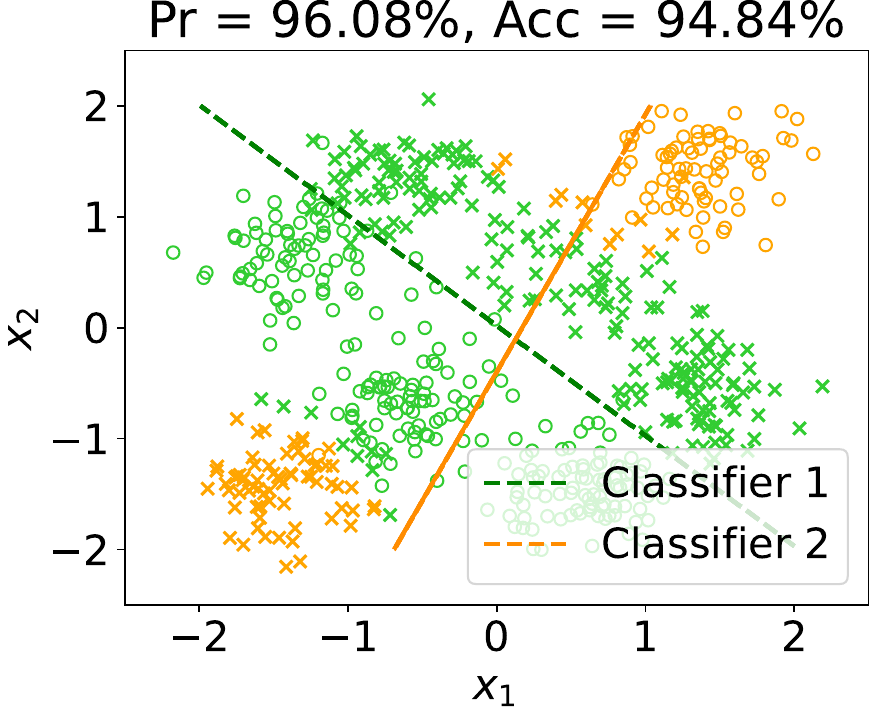}
    \caption{Finding optimal partition}
    \label{fig:syn_alg}
\end{subfigure}
\caption{Illustration of different methods on synthetic data. Circles (respectively, crosses) represent positive (respectively, negative) points. Lime (respectively, orange) points are mapped to the group associated with classifier denoted by the green (respectively, dark orange) line. Features $x_1$ and $x_2$ are correlated with $s_1$, $s_2$, and the label, so using a pooled classifier or dividing the group based on a single attribute, cannot achieve desirable performance. In contrast, our method can identify the optimal group partition, achieving the highest probability of fairness (denoted by ``Pr'') and accuracy (denoted by ``Acc'').}
\label{fig:syn}
\end{figure*}

\begin{table*}[tb]
    \centering
     \setlength{\tabcolsep}{1mm}  
    \begin{tabular}{*{8}{c}}
    \toprule
     & \multicolumn{2}{c}{DAFH (Ours)} & \multicolumn{2}{c}{Trivial partition} & \multicolumn{2}{c}{Pooled classifier} \\
     Dataset & $\Pr$ of w/o harm & Accuracy & $\Pr$ of w/o harm & Accuracy & $\Pr$ of w/o harm & Accuracy \\
     \midrule
     {Adult} & $90.70\% \pm 0.88\%$ & $\textbf{83.36\%} \pm 0.43\%$ & $91.64\% \pm 0.86\%$ & $80.70\% \pm 0.51\%$ & N/A & $81.29\% \pm 0.41\%$ \\
     {Arrest} & $\textbf{87.05\%} \pm 1.21\%$ & $\textbf{65.72\%} \pm 1.36\%$ & $78.60\% \pm 1.54\%$ & $62.01\% \pm 1.43\%$ & N/A & $63.84\% \pm 1.30\%$ \\
     {Violent} & $\textbf{91.74\%} \pm 1.49\%$ & $\textbf{81.91\%} \pm 1.53\%$ & $91.14\% \pm 1.78\%$ & $79.88\% \pm 1.94\%$ & N/A & $81.35\% \pm 1.58\%$ \\
     {German} & $\textbf{90.64\%} \pm 1.65\%$ & $\textbf{74.24\%} \pm 0.82\%$ & $78.88\% \pm 2.42\%$ & $65.92\% \pm 1.51\%$ & N/A & $70.72\% \pm 0.93\%$ \\
     {Bank} & $\textbf{97.71\%} \pm 0.34\%$ & $\textbf{90.20\%} \pm 0.10\%$ & $93.42\% \pm 0.72\%$ & $88.77\% \pm 0.55\%$ & N/A & $88.91\% \pm 0.54\%$ \\
   \bottomrule
\end{tabular}
\caption{Comparison with baselines (\textit{Trivial partition} and \textit{Pooled classifier}) on five real datasets:  
    ``$\Pr$ of w/o harm" measures the probability that samples in the testing set satisfy both rationality and envy-freeness, given the group partition and decoupled classifiers (\textit{Pooled classifiers} do not have group partition, hence we put ``N/A"). ``Accuracy" is the overall accuracy of classifiers on associated samples. Results show that we can outperform other baselines in both fairness and accuracy in most cases.}
    \label{tab:k2}
\end{table*}

\begin{figure}[tb]
    \centering
    \begin{subfigure}[t]{0.45\linewidth}
        \centering
        \includegraphics[width=\linewidth]{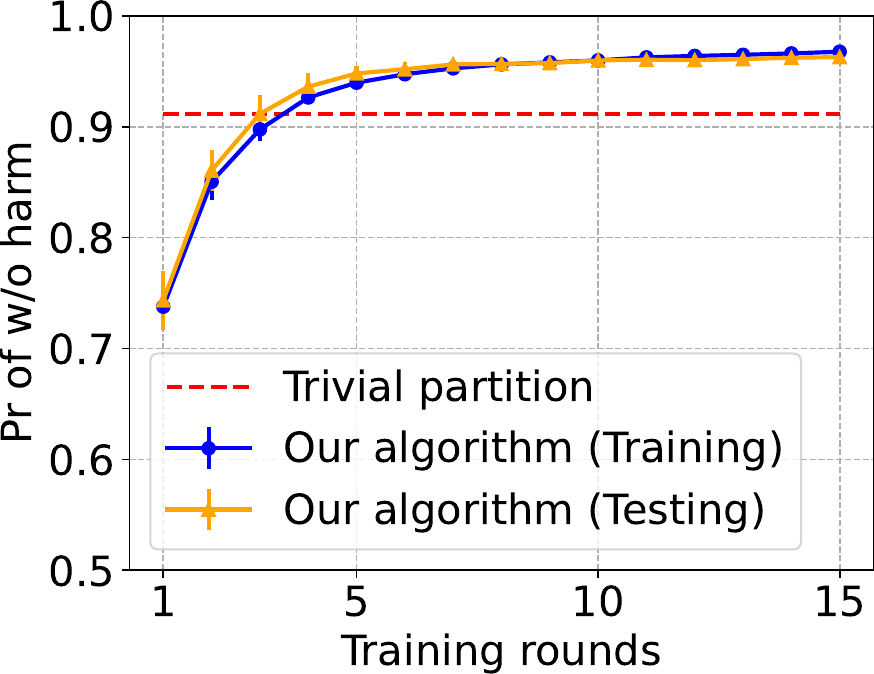}
        \caption{$\Pr$ of w/o harm}
        \label{fig:violent_harm}
    \end{subfigure}
    \hfill
    \begin{subfigure}[t]{0.45\linewidth}
        \centering
        \includegraphics[width=\linewidth]{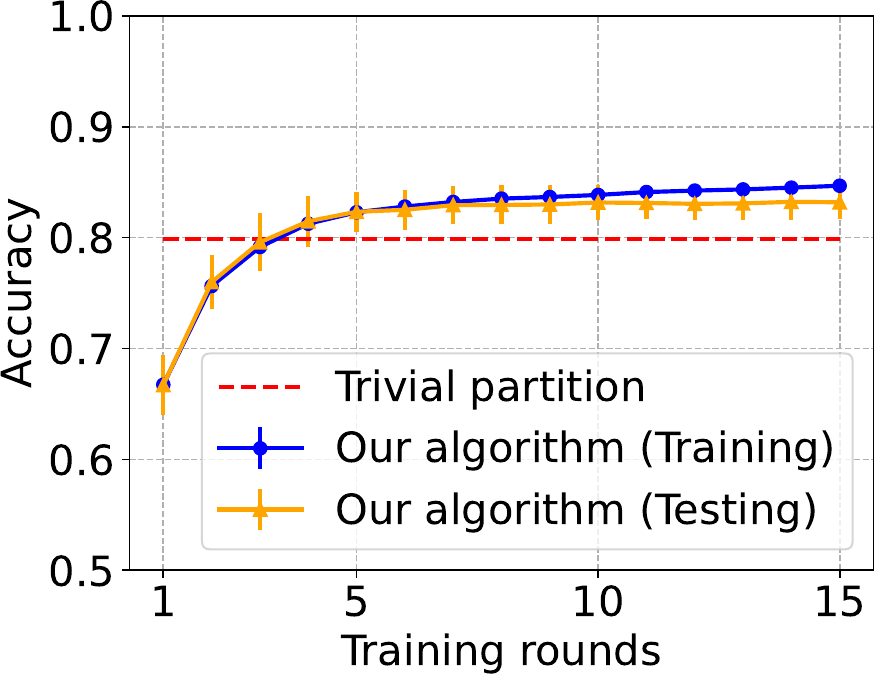}
        \caption{Accuracy}
        \label{fig:violent_acc}
    \end{subfigure}
    \caption{Convergence of our method on the Violent dataset: mean values of performance metrics during training, with error bars indicating standard deviation. Our algorithm gradually improves fairness without harm and accuracy, outperforming the \textit{Trivial partition} in fewer than 15 steps.}
    \label{fig:violent}
\end{figure}

\begin{figure}[tb]
    \centering
    \includegraphics[width=\linewidth]{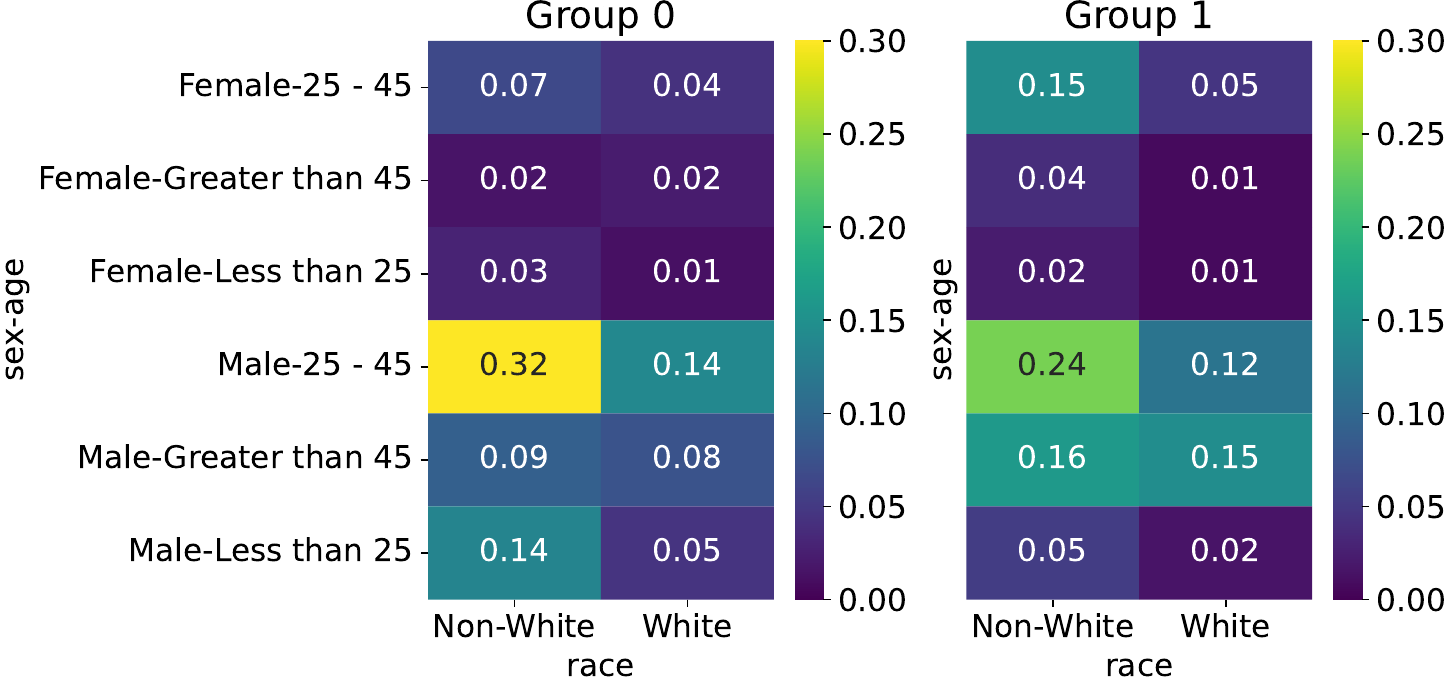}
    \caption{Physical meaning of the learned group partition on Arrest dataset: the group classifier is used to divide the dataset into two sets. The probability of different subgroups defined by 3 attributes race, sex, and age is visualized. Results show that different groups have diverse distribution on features.}
    \label{fig:physical}
\end{figure}

\begin{table}[tb]
    \centering
    \begin{tabular}{*{3}{c}}
    \toprule
      & $\Pr$ of w/o harm & Accuracy \\
      \midrule
     Manual partition & $81.39\% \pm 2.39\%$ & $62.54\% \pm 1.70\%$ \\
     Trivial partition & $78.60\% \pm 1.54\%$ & $62.01\% \pm 1.43\%$ \\
     Our algorithm & $85.72\% \pm 1.14\%$ & $64.72\% \pm 1.12\%$\\
   \bottomrule
\end{tabular}
\caption{Comparison with \textit{Manual partition} and \textit{Trivial partition} on Arrest dataset: we manually partition the dataset into 2 groups according to 3 attributes race, sex, and age, which mimics the physical meaning we observe in the group partition generated by our group classifier, and train decoupled classifiers associated with each group. Results show that manually partitioning the dataset based on our group classifier can help derive decoupled classifiers with higher performance, compared with trivially partitioning the dataset based on the chosen sensitive attribute.}
    \label{tab:manual}
\end{table}

\begin{table*}[tb]
    \centering
    \setlength{\tabcolsep}{1mm}
    \small
\renewcommand{\arraystretch}{1.2}
    \begin{tabular}{m{2.5em} m{5.5em} *{6}{c}}
    \toprule
     & & \multicolumn{2}{c}{$K=2$} & \multicolumn{2}{c}{$K=3$} & \multicolumn{2}{c}{$K=4$} \\
     Dataset & Algorithm & $\Pr$ of w/o harm & Accuracy & $\Pr$ of w/o harm & Accuracy & $\Pr$ of w/o harm & Accuracy \\
     \hline
     \multirow{2}{*}{Adult} & DAFH (Ours) 
     & $\textbf{90.70\%} \pm 0.88\%$ 
     & $\textbf{83.36\%} \pm 0.43\%$ 
     & $\textbf{94.55\%} \pm 1.31\%$ 
     & $82.82\% \pm 0.17\%$ 
     & $\textbf{93.98\%} \pm 0.47\%$ 
     & $80.98\% \pm 0.54\%$ \\
     & Clustering & $88.79\% \pm 0.58\%$ & $82.90\% \pm 0.28\%$ & $87.43\% \pm 0.53\%$ & $82.83\% \pm 0.23\%$ & $86.00\% \pm 1.19\%$ & $82.59\% \pm 0.52\%$  \\
     \hline
     \multirow{2}{*}{Arrest} & DAFH (Ours) 
     & $\textbf{87.05\%} \pm 1.21\%$ 
     & $\textbf{65.72\%} \pm 1.36\%$ 
     & $\textbf{88.06\%} \pm 1.06\%$ 
     & $\textbf{65.82\%} \pm 0.98\%$ 
     & $\textbf{88.78\%} \pm 1.18\%$ 
     & $\textbf{65.77\%} \pm 0.80\%$ \\
     & Clustering & $84.23\% \pm 1.43\%$ & $64.86\% \pm 0.89\%$ & $74.52\% \pm 2.43\%$ & $65.38\% \pm 0.81\%$ & $72.12\% \pm 1.84\%$ & $65.42\% \pm 0.38\%$  \\
     \hline
     \multirow{2}{*}{Violent} & DAFH (Ours) 
     & $\textbf{96.28\%} \pm 0.53\%$ 
     & $83.18\% \pm 1.53\%$ 
     & $\textbf{95.38\%} \pm 1.24\%$ 
     & $83.16\% \pm 1.69\%$ 
     & $\textbf{96.42\%} \pm 0.68\%$ 
     & $\textbf{84.54\%} \pm 0.74\%$ \\
     & Clustering & $87.96\% \pm 2.68\%$ & $83.82\% \pm 1.36\%$ & $85.27\% \pm 1.82\%$ & $83.58\% \pm 1.56\%$ & $84.02\% \pm 1.58\%$ & $83.64\% \pm 1.54\%$  \\
     \hline
     \multirow{2}{*}{German} & DAFH (Ours) 
     & $\textbf{90.64\%} \pm 1.65\%$ 
     & $\textbf{74.24\%} \pm 0.82\%$ 
     & $\textbf{90.00\%} \pm 2.16\%$ 
     & $\textbf{76.00\%} \pm 3.27\%$ 
     & $\textbf{89.12\%} \pm 1.32\%$ 
     & $\textbf{73.52\%} \pm 1.42\%$ \\
     & Clustering & $84.40\% \pm 1.12\%$ & $73.16\% \pm 2.38\%$ & $78.96\% \pm 2.43\%$ & $71.92\% \pm 1.51\%$ & $79.04\% \pm 2.05\%$ & $73.36\% \pm 1.69\%$  \\
     \hline
     \multirow{2}{*}{Bank} & DAFH (Ours) 
     & $\textbf{97.71\%} \pm 0.34\%$ 
     & $\textbf{90.20\%} \pm 0.10\%$ 
     & $\textbf{97.34\%} \pm 0.32\%$ 
     & $\textbf{90.06\%} \pm 0.28\%$ 
     & $\textbf{97.09\%} \pm 0.23\%$ 
     & $\textbf{90.11\%} \pm 0.30\%$ \\
     & Clustering 
     & $90.18\% \pm 0.21\%$ 
     & $88.92\% \pm 0.25\%$ 
     & $89.67\% \pm 0.25\%$ 
     & $89.07\% \pm 0.16\%$ 
     & $89.29\% \pm 0.43\%$ 
     & $88.84\% \pm 0.24\%$ \\
   \bottomrule
\end{tabular}
\caption{Comparison with \textit{Clustering} baseline on five real datasets when the number of decoupled classifiers $K=2, 3, 4$: results show that we can outperform \textit{Clustering} on both fairness and accuracy in most cases, indicating that our learned classifier can extract more information than merely the similarity in distance between different groups to improve performance.}
    \label{tab:multi_k}
\end{table*}

\begin{table*}[tb]
    \centering
    \begin{tabular}{m{3em} m{5em} *{4}{c}}
    \toprule
     Dataset & Metrics & Trivial Partition & LR-All & TreeLR & DAFH (Ours) \\
    \midrule
     \multirow{4}{*}{Adult} 
     & \# violations & 1 & 1 & 0 & \textbf{0} \\
     & max gain      &  2.9\% & 17.6\% & 4.1\% & \textbf{11.11\%} \\
     & min envy      & 2.9\% & 33.3\% & 11.8\% & 5.56\% \\
     & $\Delta$ disparity & 0.6\% & 0.4\% & -0.4\% & \textbf{-2.91\%} \\
     & \# models     & 2\ &  12  &  7  & 2 \\
     \midrule
     \multirow{4}{*}{Arrest} 
     & \# violations & 0 & 3 & 1 & \textbf{0} \\
     & max gain      & 7.7\% & 11.5\% & 9.0\% & \textbf{10\%} \\
     & min envy      & 7.7\% & 15.4\% & 12.8\% & \textbf{30\%} \\
     & $\Delta$ disparity & -0.1\% & 1.0\% & -2.7\% & \textbf{-7.9\%} \\
     & \# models     &  2  &  6  &  4  & 2 \\
     \midrule
     \multirow{4}{*}{Violent} 
     & \# violations & 1 & 0 & 2 & \textbf{0} \\
     & max gain      & 9.3\% & 13.6\% & 9.3\% & 10.61\%  \\
     & min envy      & 7.6\% & 17.0\% & 7.6\% & 13.64\%  \\
     & $\Delta$ disparity & -9.7\% & -14.0\% & -9.8\% & \textbf{-10.29\%}  \\
     & \# models     & 2 &  6  &  2  & 2 \\
   \bottomrule
\end{tabular}
\caption{Comparison with algorithms proposed in \citet{wo_harm} which rely on sensitive attributes to generate group partition. Metrics that are higher than those of other algorithms achieving the same \# of violations on each dataset are boldfaced. Results show that we can ensure no rationality or envy-freeness violation across all three datasets. We can also achieve lower $\Delta$ disparity, indicating that our algorithm can satisfy parity-based fairness criteria even without access to sensitive attributes.}
    \label{tab:compare_alg}
\end{table*}

\subsection{Dataset}
\paragraph{Synthetic dataset.} We generate the synthetic dataset as follows. Each data point has two sensitive attributes $s_1, s_2 \in \{\pm 1\}$, two features $x_1, x_2$, and label $y \in \{\pm 1\}$. The synthetic dataset comprises 20,000 samples, with each sample being randomly assigned with sensitive attributes and a label. Features are associated with both sensitive attributes and the label:
\begin{align*}
    x_1 & \sim \mathcal{N} (s_1 + \Delta s_1 s_2 y, \sigma), \\
    x_2 & \sim \mathcal{N} (s_2 + \Delta s_1 s_2 y, \sigma),
\end{align*}
where $\Delta = 0.4, \sigma = 0.3$. The synthetic dataset simulates a setting where features can act as proxy variables of sensitive attributes, and different subgroups divided by two sensitive attributes follow different distributions. A visualization of a randomly sampled subset of the dataset is depicted in Fig.~\ref{fig:syn}, where points with positive (resp. negative) labels are denoted by circles (resp. crosses), and different colors correspond to groups divided based on sensitive attributes (or the group classifier in our method).

\paragraph{Real datasets.} We use five real datasets, as detailed below:

\begin{enumerate}

\item \textbf{Adult}~\cite{Adult}: Census data to predict whether an individual's income exceeds \$50K per year. It includes 48,842 instances, each has 14 features and we choose \texttt{race} as the sensitive attribute. 

\item \textbf{Arrest}~\cite{compas}: Data used by the COMPAS algorithm to score the criminal defendants' risk of recidivism. It includes 7,214 instances, each has 52 features and we choose \texttt{race} as the sensitive attribute.

\item \textbf{Violent}~\cite{compas}: Data used by the COMPAS algorithm to predict the criminal defendants' risk of violent recidivism. It includes 4,743 instances, each has 53 features and we choose \texttt{race} as the sensitive attribute.

\item \textbf{German}~\cite{German}: Credit data used to predict individuals' credit risks (as high or low). It includes 1,000 instances, each has 20 features and we choose \texttt{sex} as the sensitive attribute. 

\item \textbf{Bank}~\cite{Bank}: The dataset focuses on bank marketing campaigns conducted via phone calls, aiming to predict whether clients will subscribe to a bank's term deposit product.  It includes 45,211 instances, each has 16 features and \texttt{age} is the sensitive attribute. 
\end{enumerate}

\subsection{Baseline Methods} We compare our method with the following baselines:  
\begin{enumerate}
\item \textit{Pooled classifier} is the classifier trained on the complete training set through structural risk minimization, and is applied to all samples in the testing set.   
\item \textit{Trivial partition} first splits each dataset according to the chosen sensitive attribute into two groups, and then trains decoupled classifiers accordingly through structural risk minimization. Typically it is restricted to the setting when the number of assigned groups $K=2$. 
\item \textit{Clustering} first partitions the dataset through a clustering algorithm (e.g., K-means in our experiments), then trains decoupled classifiers based on each partitioned set through structural risk minimization. 
\item \textit{LR-All}~\cite{wo_harm} partitions the dataset by the intersection of multiple attributes, and trains decoupled classifiers for each group.
\item \textit{TreeLR}~\cite{wo_harm} recursively grows a tree to find the group partition that best fulfills rationality and envy-freeness, with each group being assigned a decoupled classifier.
\end{enumerate}
In all baselines, we use the logistic regression as decoupled classifiers, the same as our method for fair comparison.

\subsection{Experimental Setup}
\paragraph{Implementation details.} Experiments are conducted on a personal computer with Intel Core i5-10210U CPU, 16 GB RAM, and Windows 11 operating system. Our algorithm and baselines are implemented with PyTorch 2.1.0 and scikit-learn 1.3.2. 
We use multi-layer perceptron with a hidden layer of 100 neurons as the group classifier, and logistic regression models as decoupled classifiers. Because our method does not require access to individuals' demographic information, we remove the sensitive attributes from input features when training our models. We iterate over all samples in the training set to build the computation graph of the objective function, automatically calculate the gradients for classifier parameters using autograd, and update parameters using stochastic gradient ascent to maximize the objective.

\paragraph{Evaluation metrics.} We measure the performance of our algorithm by the following metrics: 
\begin{enumerate}
\item \textbf{Probability of fairness without harm}: It is the probability that a sample satisfies both rationality and envy-freeness. Given a group classifier $\theta$, pooled classifier $h_0$, and decoupled classifiers $\{h_k\}_{k\in[K]}$, a sample $(x_i,y_i)$ satisfies both rationality and envy-freeness if the 0-1 loss of the assigned classifier $h_{\theta(x_i)}$ is not higher than the pooled classifier $h_0$ and any other decoupled classifier $h_k$, i.e., $\ell(h_{\theta(x_i)}(x_i), y_i) \leq \ell(h_k(x_i), y_i)$, $\forall k \in \{0\}\cup [K]$. The probability of fairness without harm equals the number of samples that satisfy fairness without harm divided by the size of the dataset. 
\item \textbf{Accuracy} of decoupled classifiers measures the proportion of samples in the dataset that are correctly classified by the assigned classifier, i.e, $\Pr(h_{\theta(X)}(X)=Y)$, given a group classifier $\theta$ and decoupled classifiers $\{h_k\}_{k\in[K]}$.
\item \textbf{\# of violations} is the number of groups that violates rationality, plus the number of pairwise comparison among all groups that violate envy-freeness.
\item \textbf{max gain} is the maximal difference between the true risk of pooled classifier and decoupled classifiers, $\max_{k \in [K]} ( R_k(h_0) - R_k(h_k) )$, which is the higher the better.
\item \textbf{min envy} is the maximal difference between the true risk of decoupled classifiers and the assigned classifier, $\max_{k, k^\prime \in [K]} ( R_k(h_{k^\prime}) - R_k(h_k) )$, which is the higher the better.
\item \textbf{$\Delta$ disparity} is the disparity when using decoupled classifiers minus the disparity of using pooled classifier only, where disparity is the maximal difference in accuracy between different groups, which is the lower the better.
\end{enumerate}
We randomly split each dataset into 75\%-25\% training-testing set for 5 times. For the first two metrics, we report the mean and standard deviation of metrics. For the following four metrics, we report the results when the optimal accuracy is achieved on the testing set. 
\paragraph{Hyperparameter selection.} We test all combinations of hyperparameter values shown in Table~\ref{tab:parameter}, and choose the combination that can attain the highest probability of fairness without harm on each dataset to show the results. Notably, we choose different learning rates to update parameters of group classifier and decoupled classifiers. We use the same learning rate to update parameters of the pooled classifier and decoupled classifiers in both our algorithm and other baselines. The chosen hyper-parameters on each dataset are as shown in Table~\ref{tab:parameter_ds}.

\subsection{Results}

\paragraph{Synthetic data.}

Fig.~\ref{fig:syn} evaluates our method on 
synthetic data. The performance metrics for each method are shown in the title of each figure.
In Fig.~\ref{fig:syn_pol}, a pooled classifier (green line) is applied to all samples. Since the distribution between different groups differs significantly, a single classifier cannot achieve desirable accuracy. In Fig.~\ref{fig:syn_tp1} and Fig.~\ref{fig:syn_tp2}, we partition the dataset based on sensitive attributes $s_1$ and $s_2$, respectively, and train decoupled classifiers on the divided groups. In particular, lime (resp. orange) points represent points belonging to Group 1 (resp. Group 2) and are assigned with the classifier denoted by the green (resp. dark orange) line. Since the difference in distribution results from both $s_1$ and $s_2$, partitioning groups based on only one of them leads to inferior fairness and accuracy. In contrast, by using an additional group classifier to search for the optimal partition, our algorithm shown in Fig.~\ref{fig:syn_alg} can attain the best fairness and accuracy. This further indicates that our method can effectively tackle intersectional group fairness and has greater capacity to explore data heterogeneity to enhance fairness and accuracy.

\paragraph{Real data.}
Table~\ref{tab:k2} compares our method with \textit{Trivial partition} and \textit{Pooled classifier} when the number of decoupled classifiers $K=2$. Results show that we can achieve a comparable or even higher probability of fairness without harm compared with \textit{Trivial partition} on most datasets, indicating that our algorithm can find better group partition to capture the data heterogeneity. Because the \textit{Pooled classifier} is applied to the whole dataset without any partitioned groups, we put ``N/A" in Table~\ref{tab:k2}. Moreover, our method can also achieve higher accuracy than \textit{Trivial partition} and \textit{Pooled classifier}, which means that our algorithm can increase the overall accuracy by exploring the data heterogeneity, even when we did not directly optimize for the accuracy metric.


Fig.~\ref{fig:violent} further shows the convergence of our algorithm on Violent dataset, where we use dashed horizontal lines to indicate the \textit{Trivial partition} and use solid polylines to denote ours. We report the mean values with errorbars denoting standard deviation of 5 repeated experiments. Results show that our algorithm improves fairness without harm and accuracy on both training and testing sets by maximizing the objective function, and can outperform \textit{Trivial partition} in less than 15 steps of mini-batch gradient ascent operations, with each batch containing 1,024 samples.

We explore the group partition generated by the learned group classifier based on the Arrest dataset, to find whether there are any physical meaning, as depicted in Fig.~\ref{fig:physical}. We evaluate the proportions of different subgroups of people in the two generated groups, with each subgroup being defined by the intersection of three attributes: race (`White' or `Non-White'), sex (`Female' or `Male'), and age (`Less than 25', `25 - 45', or `Greater than 45'). Notably, the proportion of non-white male less than 45 in Group 0 is especially higher than that in Group 1. The proportion of non-white female from 25 to 45 in Group 1 is especially higher than that in Group 0. 
To test whether these insights can enhance fairness and accuracy, we manually partition the Arrest dataset into two groups.
We assign all non-white males under 45 to one group, all non-white females aged 25 to 45 to another group, and randomly distribute the remaining individuals between these groups. We compare this \textit{Manual partition} algorithm with \textit{Trivial partition} baseline and also our algorithm. Results in Table~\ref{tab:manual} show that \textit{Manual partition} achieves a higher probability of fairness without harm and accuracy compared with \textit{Trivial partition}, which indicates that using a single sensitive attribute is insufficient for capturing the nuances between different groups. Our \textit{Manual partition} approach, which also considers age and sex in addition to race, underscores the importance of the intersectionality of attributes. While \textit{Manual Partition} does not outperform our algorithm, which leverages a broader set of attributes through the group classifier, it still shows significant benefits.

We also test the performance of our algorithm on the same datasets when the number of groups $K$ exceeds 2. Partitioning datasets into more than two groups based on sensitive attributes can result in some training sets being too small, potentially leading to overfitting. Therefore, we use \textit{Clustering} baseline to generate a more balanced group partition. 
Results in Table~\ref{tab:multi_k} show that our algorithm outperforms the \textit{Clustering} baseline in most cases in terms of probability of fairness without harm and accuracy, indicating that our algorithm can extract more information than merely similarity between different samples, thereby forming better group partition and increasing fairness and accuracy.

Finally, we compare our algorithm with three algorithms proposed in \citet{wo_harm}: \textit{Trivial Partition} (referred as \textit{LR (One attribute)} in~\citet{wo_harm}), \textit{LR-All}, and \textit{TreeLR}. We record the values of four metrics when each algorithm achieves optimal testing accuracy, evaluated across different numbers of models. The results in Table~\ref{tab:compare_alg} demonstrate that our approach ensures no violations of rationality or envy-freeness across all datasets, i.e., every group satisfies fairness without harm, a standard that other algorithms fail to meet. Additionally, compared to algorithms with the same number of violations, our DAFH algorithm achieves lower $\Delta$ disparity, highlighting its capability to satisfy parity-based fairness criteria. While other algorithms may achieve higher maximum gain or lower minimum envy, this is likely due to their ability to leverage sensitive attributes, enabling them to better capture groups with differing data distributions.

\section{Conclusion}\label{sec:conclusion}

In this work, we propose a novel approach to achieve fairness without harm, a preference-based fairness notion, without acccessing demographic data. Differs from parity-based fairness concept, we do not have to trade accuracy for fairness. Instead, we can train decoupled classifiers for divided groups and ensure each group prefer their own assigned classifier in terms of accuracy. Previous works require access to sensitive attributes to form group partition, which is not always feasible. Therefore, we propose the first approach that can achieve fairness without harm while not knowing sensitive attributes. Experiments show that our algorithm can achieve superior fairness and accuracy compared with baselines which have access to sensitive attributes. 

\section*{Acknowledgments}
This material is based upon work supported by the U.S. National Science Foundation under award IIS2202699, IIS-2416895, IIS-2301599, and CMMI-2301601, and by OSU President’s Research Excellence Accelerator Grant, and grants from the Ohio State University’s Translational Data Analytics Institute and
College of Engineering Strategic Research Initiative.
\bibliography{ref}

\end{document}